\documentclass{article} % For LaTeX2e
\usepackage{natbib}
\usepackage{macros}
\usepackage{authblk}
\usepackage{multirow}
\usepackage{times}
\usepackage{helvet}
\usepackage{courier}
\tikzset{roads/.style={line width=0.2cm}}

\title{Fast gradient descent for drifting least squares regression, with application to bandits}
\author[1]{Nathaniel Korda \thanks{nathaniel.korda@eng.ox.ac.uk}}
\author[2]{Prashanth L A \thanks{prashanth.la@inria.fr}}
\author[2]{R\'emi Munos \thanks{remi.munos@inria.fr}}
\affil[1]{\small Oxford University, UNITED KINGDOM.}
\affil[2]{\small INRIA Lille - Nord Europe, Team SequeL, FRANCE.}
\date{}
\begin{document}
\maketitle

%%%%%%%%%%%%%%%%%%%%%%%%%%%%%%%%%%%%%%%%%%%%%%%%%%%%%%%%%%%%%%
%%%%%%%%%%%%%%%%%%%%%%%%%%%%%%%%%%%%%%%%%%%%%%%%%%%%%%%%%%%%%%
\begin{abstract}
 Online learning algorithms require to often recompute least squares regression estimates of parameters. We study improving the computational complexity of such algorithms by using stochastic gradient descent (SGD) type schemes in place of classic regression solvers. We show that SGD schemes efficiently track the true solutions of the regression problems, even in the presence of a drift.
This finding coupled with an $O(d)$ improvement in complexity, where $d$ is the dimension of the data, make them attractive for implementation in the \textit{big data} settings.
In the case when strong convexity in the regression problem is guaranteed, we provide bounds on the error both in expectation and high probability (the latter is often needed to provide theoretical guarantees for higher level algorithms), despite the drifting least squares solution.
As an example of this case we prove that the regret performance of an SGD version of the PEGE linear bandit algorithm is worse than that of PEGE itself only by a factor of $O(\log^4 n)$.  
When strong convexity of the regression problem cannot be guaranteed, we investigate using an adaptive regularisation.
We make an empirical study of an adaptively regularised, SGD version of LinUCB in a news article recommendation application, which uses the large scale news recommendation dataset from Yahoo! front page.
These experiments show a large gain in computational complexity and a consistently low tracking error.
\end{abstract}

%%%%%%%%%%%%%%%%%%%%%%%%%%%%%%%%%%%%%%%%%%%%%%%%%%%%%%%%%%%%%%
%%%%%%%%%%%%%%%%%%%%%%%%%%%%%%%%%%%%%%%%%%%%%%%%%%%%%%%%%%%%%%
\section{Introduction}
\noindent Often in learning algorithms an unknown parameter must be estimated from data arriving sequentially in pairs, $(x_n,y_n)$.
We consider  settings where the points $x_n$ are chosen by a higher level algorithm and the outputs $y_n$ satisfy the dynamics $y_n = x_n\tr \theta^* + \xi_n$,
where $\xi_n$ is i.i.d., zero-mean noise, and $\theta^*$ is the unknown parameter
(the flow diagram, Fig. \ref{fig:ttalgo}, illustrates this setting).
Typically, in such cases an ordinary least squares (OLS) estimate is used for $\theta^*$, and finding this estimate is often the most computationally intensive part of the higher level algorithm.
The solution to the least squares regression problem is defined as
\begin{align}\label{eq:ols}
 \hat \theta_n = \argmin_{\theta} \left\{F_n(\theta):= \dfrac{1}{2} \sum\limits_{i=1}^{n} (y_i - \theta\tr x_i)^2\right\}.
\end{align}
That $\hat \theta_n =\bar  A^{-1}_n \bar b_n$, where 
$\bar A_n =  n^{-1} \sum_{i=1}^{n} x_i x_i\tr$
and $\bar b_n =  n^{-1} \sum_{i=1}^{n} x_i y_i$,
is well-known.
Assuming that the features $x_i$ evolve in a compact subset $\cD$ of $\R^d$, the complexity of solving \eqref{eq:ols} with the above approach is $O(d^2)$, where the inverse of $\bar A_n$ is computed iteratively using the Sherman-Morrison lemma. Using the Strassen algorithm or the Coppersmith-Winograd algorithm gives a complexity of $O(d^{2.807})$ and $O(d^{2.375})$ respectively. In addition, there is an order $O(d^2 n)$ complexity for computing $\bar A_n$.

Unlike the traditional gradient descent (GD) setting where the pairs $(x_n,y_n)$ are samples drawn from some unknown joint probability distribution, we assume that the samples, $x_n$, are chosen by a higher level learning algorithm, and the problem is to find a good enough approximation to $\theta^*$ for its purposes, given these non-i.i.d. samples.
This poses a new difficulty in applying GD schemes directly, and we outline two well-known solutions to this problem in the following.

As illustrated in Fig. \ref{fig:ttalgo}, the classic SGD algorithm  operates by maintaining an iterate $\theta_n$ that is updated as follows: Choose a random sample $(x_{i_n},y_{i_n})$, where $i_n$ is picked uniformly at random in $\{1,\dots,n\}$ and update
\begin{align}
 \label{eq:fOLS-GD-update}
\theta_n = \theta_{n-1} + \gamma_n (y_{i_n} - \theta_{n-1}\tr x_{i_n})x_{i_n},
\end{align}
(The sequence of stepsizes $\gamma_n$ is chosen in advance, see assumption (A4) below for details.)
The complexity of each iteration above is $O(d)$, while traditional approaches giving the exact solution, such as using the Sherman-Morrison lemma, incur a cost of at least $O(d^2)$ per iteration.
We shall refer to SGD applied to our setting as fOLS-GD (fast Online Least Squares - Gradient Descent).

Unlike previous works which analyse the above SGD algorithm in a batch setting, we consider a drifting least squares setting.
In particular, at each instant $n$, the SGD update is required to track the minimiser $\hat\theta_n$ of the function $F_n(\cdot)$, as $n$ increases.
The practical advantage of such an approach is to replace the costly inversion of the $\bar A_n$ matrix with an efficient iterative scheme.
However, from a theoretical standpoint, fOLS-GD has to grapple with the drift error, $\| \hat\theta_n - \hat\theta_{n-1}\|_2$, that accumulates with time.

Under a minimum eigenvalue assumption on the matrices $\bar A_n$, we find that ordinary SGD is sufficient to mitigate the effects of drift in $\hat\theta_n$.
In this case, we provide bounds both in expectation and in high probability on the approximation error $\theta_n - \hat\theta_n$, where $\theta_n$ is the fOLS-GD iterate at instant $n$ (see Theorem \ref{cor:fOLS-GD}). Such bounds are essential for giving theoretical guarantees when using fOLS-GD as a subroutine to replace the matrix inversion approach to the regression problem in a higher level learning algorithm. 

To cope with situations where the minimum eigenvalue assumption of the $\bar A_n$ matrix cannot be guaranteed by the higher level algorithm we propose adding an adaptive regularisation: since our data is growing with time we introduce a regularisation parameter, $\lambda_n$, that adapts to the sample size $n$ as follows:
\begin{align}
\label{eq:reg-ls}
\ttheta_n: =  \arg\min_\theta \frac{1}{2n} \sum_{i=1}^n (y_i - \theta\tr x_i)^2 +\lambda_n \l \theta \r^2.
\end{align}
This algorithm, which we henceforth refer to as fRLS-GD (fast Regularised online Least Squares - Gradient Descent), tracks the regression solutions, $\ttheta_n$ and operates in a manner similar to fOLS-GD (see Fig. \ref{fig:ttalgo}) except that we factor in the regularisation parameter $\lambda_n$ into the update rule:
\begin{align}
 \label{eq:fRLS-GD-update}
 \theta_n = \theta_{n-1} + \gamma_n( (y_{i_n} - \theta_{n-1}\tr x_{i_n})x_{i_n}- \lambda_n\theta_{n-1}).
\end{align}
Unlike fOLS-GD, the above algorithm will suffer a bias due to the adaptive regularisation and it is difficult to provide bounds in theory owing to the bias error (see discussion after Eq. \eqref{eq:reg-decomp}). However, we demonstrate empirically that fRLS-GD is able to consistently track the true RLS solutions, when used within a higher level algorithm. The advantage, however, of using fRLS-GD in place of classic RLS solvers is that it results in significant computational gains.

As examples of higher level learning algorithms using regression as a subroutine, we consider two linear bandit algorithms.
% and the ConfidenceBall algorithm of \cite{dani2008stochastic}.
In a linear bandit problem the values $x_n$ represent actions taken by an agent and the values $y_n = x_n\tr\theta^* + \xi_n$ are interpreted as random rewards, with unknown parameter $\theta^*$. At each time the agent can choose to take any action $x\in\cD$, where $\cD$ is some compact subset of $\R^d$, and the agent's goal is to maximise the expected sum of rewards.
This goal would be achieved by choosing $x_n = x^* :=\argmin_x\{x\tr \theta^*\}$, $\forall n$.
However, since one does not know $\theta^*$ one needs to estimate it, and a tradeoff appears between sampling pairs $(x_n,y_n)$ that will improve the estimate, and gaining the best short term rewards possible by exploiting the current information available.
Typically the performance of a bandit algorithm is measured by its \emph{expected cumulative regret}: $\cR_n = \sum_{i=1}^n (x^* - x_i)\tr \theta^*$.

First, we consider the PEGE algorithm for linear bandits proposed by \cite{Tsitsiklis2010}.
This algorithm is designed for action sets $\cD$ satisfying a strong convexity property (see assumption (A4)),
and so we can provide a computationally efficient variant of PEGE where the fOLS-GD iterate, $\theta_n$, is used in place of the OLS estimate, $\hat\theta_n$, in each iteration $n$ of PEGE.
PEGE splits time into exploration and exploitation phases.
During the exploitation phases the algorithm acts greedily using OLS estimates of $\theta^*$ calculated from data gathered during the exploration phases.
During the exploration phases data is gathered in such a way that  the smallest eigenvalues of $\bar A_n$ matrices are uniformly bounded for all $n$.
The regret performance of this algorithm is $O(dn^{1/2})$, and we establish that our variant using fOLS-GD as a subroutine achieves an improvement of order $O(d)$ in complexity, while suffering a loss of only $O(\log^4 n)$ in the regret performance.

\tikzstyle{block} = [draw, fill=white, rectangle,
    minimum height=3em, minimum width=6em]
\tikzstyle{sum} = [draw, fill=white, circle, node distance=1cm]
\tikzstyle{input} = [coordinate]
\tikzstyle{output} = [coordinate]
\tikzstyle{pinstyle} = [pin edge={to-,thin,black}]

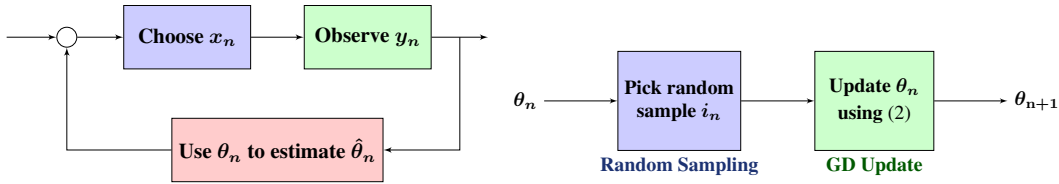
\begin{figure}
\centering
\begin{tabular}{cc}
\scalebox{0.8}{\begin{tikzpicture}[auto, node distance=2cm,>=latex']
    \node [input, name=input] {};
    \node [sum, right of=input] (sum) {};
    \node [block, fill=blue!20,right of=sum] (controller) {\textbf{Choose} $\boldsymbol{x_n}$};
    \node [block, fill=green!20,right of=controller,
            node distance=3cm] (system) {\textbf{Observe $\boldsymbol{y_n}$}};
    % We draw an edge between the controller and system block to
    % calculate the coordinate u. We need it to place the measurement block.
    \draw [->] (controller) -- node[name=u] {} (system);
    \node [output, right of=system] (output) {};
    \node [block, fill=red!20,below of=u] (measurements) { \textbf{Use} {\boldmath$\theta_n$}  \textbf{to estimate} {\boldmath $ \boldsymbol{\hat\theta_n}$}  };
    %
    % Once the nodes are placed, connecting them is easy.
    \draw [draw,->] (input) -- node {} (sum);
    \draw [->] (sum) -- node {} (controller);
    \draw [->] (system) -- node [name=y] {}(output);
    \draw [->] (y) |- (measurements);
    \draw [->] (measurements) -| node[pos=0.99] {}
        node [near end] {} (sum);
\end{tikzpicture}}
&
\tikzstyle{block} = [draw, fill=white, rectangle,
   minimum height=5em, minimum width=6em]
\tikzstyle{sum} = [draw, fill=white, circle, node distance=1cm]
\tikzstyle{input} = [coordinate]
\tikzstyle{output} = [coordinate]
\tikzstyle{pinstyle} = [pin edge={to-,thin,black}]
\hspace{-2em}
\scalebox{0.75}{\begin{tikzpicture}[auto, node distance=2cm,>=latex']
% We start by placing the blocks
\node (theta) {$\boldsymbol{\theta_n}$};
\node [block, fill=blue!20, right=1.3cm of theta,label=below:{\color{bleu2}\bf Random Sampling},align=center] (sample) {\textbf{Pick random}\\[1ex] \textbf{sample }$\boldsymbol{i_n}$}; 
\node [block, fill=green!20, right=1.3cm of sample,label=below:{\color{darkgreen}\bf GD Update},align=center] (update) {\textbf{Update} $\boldsymbol{\theta_n}$ \\[1ex]\textbf{using }\eqref{eq:fOLS-GD-update}};
\node [right=1.3cm of update] (end) {$\boldsymbol{\mathbf{\theta_{n+1}}}$};
\draw [->] (theta) --  (sample);
\draw [->] (sample) -- (update);
\draw [->] (update) -- (end);
\end{tikzpicture}}
\end{tabular}
\caption{Estimating OLS $\hat\theta_n$ using online SGD within a higher-level machine learning algorithm}
\label{fig:ttalgo}
\end{figure}

Second, we consider the LinUCB algorithm proposed  \cite{li2010contextual}.
Here we investigate computationally efficient variants of LinUCB. We begin by replacing the OLS estimate with an fRLS-GD iterate, and then compare this to two other state-of-the-art OLS schemes from \cite{johnson2013accelerating}  and \cite{roux2012stochastic}.
The LinUCB algorithm is designed for situations where at each time, $n$, the agent can choose only from a given, finite subset of $\cD$.
The algorithm then calculates an optimistic upper confidence bound (UCB) for the mean reward associated with each feature, and then selects a feature greedily with respect to this UCB\footnote{Calculating the UCBs is in itself an NP-hard problem for all but simple decision sets. However, we alleviate this problem by considering a setting where the sets of arms at each time instant is a finite subset of $\cD$.}.

LinUCB, however, cannot guarantee that the minimum eigenvalue of $\bar A_n$ matrices is uniformly bounded, and so we apply fRLS-GD in place of fOLS-GD.
Moreover, we devise a simple GD procedure for estimating the confidence term of the UCB for each arm. 
The resulting LinUCB variant achieves an $O(d)$ improvement in complexity over regular LinUCB. From the numerical experiments, we observe that the fRLS-GD iterate as well as SVRG \cite{johnson2013accelerating} and SAG \cite{roux2012stochastic} variants consistently track the true RLS solutions in each iteration of LinUCB, while the runtime gains are significant.

%%%%%%%%%%%%%%%%%%%%%%%%%%%%%%%%%%%%%%%%%%%%%%%%%%%%%%%%%%%%%%
\paragraph{Related work.}
SGD is a popular approach for optimizing a function given noisy observations, while incurring low computational complexity. Non-asymptotic bounds in expectation for SGD schemes have been provided by \cite{bach2011non}. In the machine learning community, several algorithms have been proposed for minimising the regret, for instance, \cite{zinkevich2003online,hazan2011beyond,rakhlin2011making} and these can be converted to find the minimiser of a (usually convex) function. A closely related field is stochastic approximation (SA), and concentration bounds for SA algorithms have been provided by \cite{frikha2012concentration}. Adaptive regularisation in the context of least squares regression has been analysed in \cite{tarres2011online}. For recent algorithmic improvements to solving batch problems, the reader is referred to the works of \cite{roux2012stochastic,shalev2012stochastic,johnson2013accelerating}.  

In general, none of the schemes proposed above are directly applicable in our setting due to two difficulties: \\
\begin{inparaenum}[\bfseries(i)]
\item our data $\{(x_i,y_i)\}_{i=1}^n$ do not arrive from a distribution, but instead are chosen by a higher level algorithm, and\\
\item an efficient online scheme is required to track the solution of a least squares regression problem with a growing data set, and thus a drifting target.\\
\end{inparaenum}
Earlier works solve one batch problem or a sequence of batch problems with data arriving from a distribution. On the other hand, we consider a drifting regression setting and study low complexity SGD schemes. For a strongly convex setting, we are able to provide theoretical guarantees, while for a non-strongly convex setting, we obtain encouraging results empirically.

%%%%%%%%%%%%%%%%%%%%%%%%%%%%%%%%%%%%%%%%%%%%%%%%%%%%%%%%%%%%%%
%%%%%%%%%%%%%%%%%%%%%%%%%%%%%%%%%%%%%%%%%%%%%%%%%%%%%%%%%%%%%%
\section{Gradient Descent for Online Least Squares}
\label{sec:fOLS-GD}
\label{sec:online}
In this section, we present the results for the fOLS-GD procedure outlined earlier. Recall that fOLS-GD tracks the OLS estimate $\hat\theta_n:=\min_\theta \frac{1}{2} \sum_{i=1}^n (y_i - \theta\tr x_i)^2$ as the samples $(x_i,y_i)$ arrive sequentially (see Fig. \ref{fig:ttalgo}) and updates the parameter as follows: Fix $\theta_0$ arbitrarily and update
\begin{align}
 \label{eq:fOLS-GD-update-again}
\theta_n = \theta_{n-1} + \gamma_n (y_{i_n} - \theta_{n-1}\tr x_{i_n})x_{i_n},
\end{align}
where $i_n\sim\cU(\{1,\dots,n\})$. Here $\cU(S)$ denotes the uniform distribution on the set $S$, and so the samples $(x_{i_n},y_{i_n})$ passed to \eqref{eq:fOLS-GD-update-again} are chosen uniformly randomly from the set $\{(x_1,y_1),\ldots,(x_n,y_n)\}$.

\paragraph{Results}
We make the following assumptions:\\
\begin{inparaenum}[\bfseries({A}1)]
\item $\sum_n\gamma_n = \infty$ and $\sum_n\gamma^2_n < \infty$.\\
\item Boundedness of $x_n$, i.e., $\sup_n \l x_n \r \le 1$.\\
\item The noise $\{\xi_n\}$ is i.i.d. and $|\xi_n|\le 1,\forall n$.\\
\item For all $n$ larger than some initial $n_0$, ${\lambda_{\min}(\bar A_n)}\ge \mu$, where $\lambda_{\min}(\cdot)$ denotes the smallest eigenvalue of a matrix.\\
\end{inparaenum}
The first assumption is a standard one for the step sizes of SGD, and, more generally, stochastic approximation schemes.
While the next two assumptions are standard in the context of least squares, the last assumption is made necessary due to the fact that we do not regularise the problem. Initially $A_n$ may not invertible, and hence the condition can only reasonably hold after some initial time $n_0$. 

In the following, we bound the approximation error $\|\theta_n-\hat\theta_n\|$ of fOLS-GD, both in high probability as well as in expectation.
\begin{theorem}
\label{cor:fOLS-GD}
Under (A2)-(A4), with $\gamma_n= c/(4(c+n))$ and $\mu c / 4 \in(2/3, 1)$,  for any $\delta >0$ and $n>n_0$,
\begin{align}
\E \left( \l \theta_n - \hat\theta_n \r \right) \le \frac{K_1(n)}{\sqrt{n+c}},\text{ and }
P\left( \l \theta_n - \hat\theta_n \r  \le \frac{K_2(n)}{\sqrt{n+c}}\right) \ge 1 - \delta \label{eq:l2-prob-bound},
\end{align}
where
\begin{align*}
K_1(n) = \frac{\| \theta_{n_0} - \theta^* \|\ln(n_0)}{(n+c)^{\mu c / 4}} + \sqrt{h(n)} + \frac{\sqrt{2}+\sqrt{\mu\beta_{n+c}}}{\mu},\quad
K_2(n) = \sqrt{2 K_{\mu, c} \log\frac{1}{\delta}}+ K_1(n),
\end{align*}
\begin{align*}
K_{\mu,c}=c^2/\left[16\left(1-2(1-3\mu c/16)\right)\right], \quad
\beta_n = \max\left(128 d \log n \log{n^2 \delta^{-1}}, \left(2\log{n^2 \delta^{-1}}\right)^2 \right),
\end{align*}
and $h(k) = 2\left[1 + 2(\l \theta_0 - \theta^*\r +\log k)^2\right]$.
\end{theorem}
\paragraph{Proof Sketch}
In order to prove the bound in expectation, following the proof scheme of \cite{frikha2012concentration}, we expand the error at time $n$ into an initial error term, a (martingale) sampling error term, and a drift error term as follows:
\begin{align*}
\theta_n - \hat\theta_{n}
	&=  \theta_n - \hat\theta_{n-1} + \hat\theta_{n-1} - \hat\theta_{n}
	= \theta_{n-1} - \hat\theta_{n-1} + \hat\theta_{n-1} - \hat\theta_{n} + \gamma_n (y_{i_n} - \theta_{n-1}\tr x_{i_n})x_{i_n}\\
	&= \underbrace{\frac{\Pi_n}{\Pi_{n_0}} (\theta_{n_0} - \theta^*)}_{\text{Initial Error}}
		+ \sum\limits_{k=1}^{n}\bigg[ \underbrace{\gamma_k \frac{\Pi_n}{\Pi_k} \Delta \tM_k}_{\text{Sampling Error}}
		-\underbrace{\frac{\Pi_n}{\Pi_k} (\hat\theta_{k} - \hat\theta_{k-1})}_{\text{Drift Error}}\bigg],
\end{align*}
where $\Pi_n:=\prod_{k=1}^{n}\left(I - \gamma_k \bar A_k\right)$, and $\Delta \tM_k$ is a martingale difference (see the Appendix \ref{sec:analysis} below for details).

The initial and sampling errors appear as in previous works on SGD (cf. \cite{frikha2012concentration} and  \cite{bach2011non}), and can be treated similarly, except that here we can make all the constants explicit, using the specific form of the update rule, and also that $\l \Pi_n \Pi_k^{-1} \r \le \e (\Gamma_n - \Gamma_k)$, where $\Gamma_n: = \sum_{i = 1}^n \gamma_i$.
In this way, choosing the step sequence as in the Theorem statement, we derive the first and second terms of $K_1(n)$.

The drift error, however, is not present in previous works, and comes from the fact that the target of the algorithm, $\hat\theta_n$, is drifting over time. To control it we note that
\begin{align*}
&\left( \nabla F_n(\hat\theta_n) = 0 = \nabla F_{n-1}(\hat\theta_{n-1})  \right)\implies\left( \hat\theta_{n-1} - \hat\theta_n
=\left(\xi_n A_{n-1}^{-1}  - (x_n\tr(\hat\theta_n -\theta^*))A_{n-1}^{-1}\right)x_n\right).
\end{align*}
Thus it is controlled by the convergence of the least squares solution $\hat\theta_n$ to $\theta^*$. Adapting a confidence ball result from \cite{dani2008stochastic}, we derive the third term of $K_1$.

Having bounded the mean error, we can bound separately the deviation of the error from its mean. To do this, following \cite{frikha2012concentration}, we decompose $\| \theta_n - \hat\theta_{n} \|^2 - E \| \theta_n - \hat\theta_{n} \|^2$ into a sum of martingale differences as follows: Let $\cH_n$ denoting the sigma-field $\sigma(i_1,\ldots,i_n)$.
\begin{align}
 \| \theta_n - \hat\theta_{n} \|_2 - \E \| \theta_n - \hat\theta_{n} \|_2 = & \sum\limits_{i=1}^{n} g_i - \E [ g_i \left| \cH_{i-1}\right.], 
\end{align}
where $g_i = \E [\| \theta_n - \hat\theta_{n}\|_2 \left| \theta_i \right.]$.
Next, we establish that the functions $g_i$ are Lipschitz continuous in the noise $\xi_i$, with Lipschitz constants $L_i$. Unlike in \cite{frikha2012concentration} we use the exact form of the update to derive the exact constants $L_i$. The final step of the proof is to invoke a standard martingale concentration bound. A complete proof is contained in the Appendix \ref{sec:analysis}.
\paragraph{Rates} With the step-sizes specified in Theorem \ref{cor:fOLS-GD}, we see that the initial error is forgotten exponentially faster than the drift and sampling errors, which vanish at the rate $O\left(n^{-1/2}\right)$. The rate derived in Theorem \ref{cor:fOLS-GD} matches the asymptotically optimal convergence rate for SGD type schemes that do not involve a drifting target (see \cite{nemirovsky1983problem}).
\paragraph{Dependence on $d$} The dependence of the rate derived above on the dimension $d$ of $x_i$ is indirect, through the strong convexity constant $\mu$. For example, in the application to strongly-convex linear bandits in the next section, after an initial $d$ steps, the strong convexity constant is known and is of order $\mu = \Omega(1/d)$, and so the derived rate has a linear dependence on $d$. 
\paragraph{Iterate Averaging}
Ensuring the optimal rate for fOLS-GD requires knowledge of the strong convexity constant $\mu$. In our application to linear bandits in the next section we know this constant. However, we can use Polyak averaging together with the step size $\gamma_n = cn^{-\alpha}$ to arrive at an optimal rate independent of the choice of $c$. 

%%%%%%%%%%%%%%%%%%%%%%%%%%%%%%%%%%%%%%%%%%%%%%%%%%%%%%%%%%%%%%
%%%%%%%%%%%%%%%%%%%%%%%%%%%%%%%%%%%%%%%%%%%%%%%%%%%%%%%%%%%%%%
\section{Strongly Convex Bandits with Online GD}
\label{sec:strongly-convex-bandits}
\paragraph{Background for PEGE}
In this section, we assume that $\cD$ is a strongly convex set and the ``best action" function, denoted by $G(\theta):= \argmin_{x\in \cD}\{\theta\tr x\}$, is assumed to be smooth in the unknown parameter $\theta$ that governs the losses of the bandit algorithm (see (A5) below). 
PEGE of \cite{Tsitsiklis2010} is a well-known algorithm in this setting.
Recall from the introduction that it gathers data and computes least squares estimates of $\theta^*$ during exploration phases, between which it exploits the estimates during exploitation phases of growing length. 
Since strong convexity in the regression problem is guaranteed by the algorithm we propose a variant of PEGE which replaces the calculation of the least squares estimate with fOLS-GD (see Algorithm \ref{alg:PEGEwithGD}).
Whereas, after $m$ exploration phases, PEGE has incurred a complexity of $O(m d^3)$, our algorithm has incurred an improved  complexity of only $O(md^2)$.

\algblock{PEGEval}{EndPEGEval}
\algnewcommand\algorithmicPEGEval{\textbf{\em Exploration Phase}}
 \algnewcommand\algorithmicendPEGEval{}
\algrenewtext{PEGEval}[1]{\algorithmicPEGEval\ #1}
\algrenewtext{EndPEGEval}{\algorithmicendPEGEval}

\algblock{PIPEGEmp}{EndPIPEGEmp}
\algnewcommand\algorithmicPIPEGEmp{\textbf{\em Exploitation Phase}}
 \algnewcommand\algorithmicendPIPEGEmp{}
\algrenewtext{PIPEGEmp}[1]{\algorithmicPIPEGEmp\ #1}
\algrenewtext{EndPIPEGEmp}{\algorithmicendPIPEGEmp}

\algtext*{EndPEGEval}
\algtext*{EndPIPEGEmp}

\begin{algorithm}[t]  
\caption{fPEGE-GD}
\label{alg:PEGEwithGD}
\begin{algorithmic}
\State {\bfseries Input:}
Get a basis $\{b_1,\dots,b_d\}\in D$ for $\R^d$. Set $c = 4d/(3\lambda_{\min}(\sum_{i = 1}^d b_i b_i\tr))$ and $\theta_0 =  0$.
\For{$m = 1,2,\ldots$}
%     \PEGEval
\State \textit{Exploration Phase}
	\For{$n = (m-1)d$ {\bfseries to} $md - 1$}
		\State \quad   Choose arm $x_{n} = b_{n \text{ mod } md}$ and observe $y_{n}$.
		\State \quad   Update $\theta$ as follows:
	  				$\theta_{n} = \theta_{n-1} + \frac{c}{n}( (y_{j} - \theta_{n-1}\tr x_{j})x_{j})$,
				where $ j\sim\cU({1,\dots,n})$.
	\EndFor
%     \EndPEGEval
\State \textit{Exploitation Phase}
%     \PIPEGEmp
	  \State  Find $x = G(\theta_{md}) := \argmin_{x\in D}\{\theta_{md}\tr x\}$.
	  \State  Choose arm $x$ $m$ times consecutively.
%     \EndPIPEGEmp
  \EndFor
\end{algorithmic}
\end{algorithm}

\paragraph{Results}
We require the following extra assumptions from \cite{Tsitsiklis2010}:\\
\begin{inparaenum}
\item[\bfseries{({A}4')}] A basis $\{b_1,\dots,b_d\}\in\cD$ for $\R^d$ is known to the algorithm.\\
\item[\bfseries{({A}5)}] The function $ G(\theta)$ is $J$-Lipschitz.\\
\end{inparaenum}
The assumption (A5) is satisfied, for example, when $\cD$ is the unit sphere. However it is not satisfied when $\cD$ is discrete.
The main result that bounds the regret of fPEGE-GD is given below. The final bound is worse than that for PEGE by only a factor of $O(\log^4(n))$:
\begin{theorem}
\label{thm:pege-regret}
Let $ \lambda_{PEGE} := \lambda_{\min}(\sum_{i = 1}^d b_i b_i\tr)$. Under the assumptions (A2), (A3), (A4'), and (A5) and with  stepsize $\gamma_n = c/(4(c+n))$, where $c/(4\lambda_{PEGE})\in (2/3, 1)$, the cumulative regret $R_n$ of fPEGE-GD is bounded as follows:
\begin{align*}
R_n \le C K_1(n)^2d^{-1} (\l\theta^*\r+ \l\theta^*\r^{-1})n^{1/2},
\end{align*}
where $C$ is a constant depending on $\lambda_{PEGE}$ and $J$, and $K_1(n) = O(d\log^2(n))$.
\end{theorem}
\begin{proof}
We have
$
\lambda_{\min}(\bar A_n) \ge \lambda_{\min}\left(\frac{(n \text{ mod } d)\sum_{i = 1}^d b_i b_i\tr}{[(n \text{ mod } d)+1]d}\right) \ge \lambda_{\min}\left(\sum_{i = 1}^d b_i b_i\tr\right)/(2d)
$
for all $n>d$.
So, choosing $c$ as in the theorem statement, we can apply Theorem \ref{cor:fOLS-GD} to get:
\begin{align}\label{eq:exp-bound-pege}
\E \l \theta_n - \theta^* \r^2 \le K_1(n)^2/(dn).
\end{align}
Now to complete the proof we only need to reprove Lemma 3.6 of \cite{Tsitsiklis2010}, which states that for all $n\ge d$, 
$\E \l \theta^* (G(\theta^*) - G(\theta_{md}))\r \le \frac{K_1(n)}{dm \l \theta^* \r}$:
\begin{align*}
\| \theta^* (G(\theta^*) - G(\theta_{md}))\|_2
=& \big\| (\theta^* - \theta_{md})\tr G(\theta^*) + (G(\theta^*) - G(\theta_{md}) )\tr \theta_{md} + (\theta_{md} - \theta^*)G(\theta_{md})\big\|_2 \\
 \le&\l (\theta^* - \theta_{md})\tr(G(\theta^*) - G(\theta_{md}) ) \r \le \frac{2J \l \theta^* - \theta_{md}\r^2}{\l \theta^*\r},
\end{align*}
where the second inequality we have used that $G(\theta) = G(a\theta)$ for all $a>0$, (A5), and Lemma 3.5 of \cite{Tsitsiklis2010}.

The rest of the proof follows that of Theorem 3.1 of \cite{Tsitsiklis2010}.
\end{proof}

%%%%%%%%%%%%%%%%%%%%%%%%%%%%%%%%%%%%%%%%%%%%%%%%%%%%%%%%
%%%%%%%%%%%%%%%%%%%%%%%%%%%%%%%%%%%%%%%%%%%%%%%%%%%%%%%%
\section{Online GD for Regularized Least Squares}
\label{sec:fRLS-GD}
Ideally an online algorithm would not need to satisfy an assumption such as (A4).
Perhaps the most obvious way to obviate (A4) is to regularise.
In an offline setting the natural regularisation parameter would be $\lambda/T$ for some $\lambda>0$, where $T$ is the size of the batch.
However in an online setting we envisage obtaining arbitrary amounts of information, and so we need to regularize adaptively at each time step by $\lambda_n$ (see \eqref{eq:reg-ls}).
As outlined earlier, the fRLS-GD algorithm attempts to shadow the solutions $\ttheta_n$ of the $\lambda_n$-regularised problem, using the following iterate update: 
\begin{align}
 \label{eq:roarlsa-update}
 \theta_n = \theta_{n-1} + \gamma_n( (y_{i_n} - \theta_{n-1}\tr x_{i_n})x_{i_n}- \lambda_n\theta_{n-1}),
\end{align}
where $i_n\sim\cU({1,\dots,n})$. 

\paragraph{Discussion}
It is interesting to note that the analysis in Theorem \ref{cor:fOLS-GD} does not generalise to this setting. Following the same argument as for the proof of Theorem \ref{cor:fOLS-GD} will lead to the iteration:
\begin{align}
\theta_n - \ttheta_{n}
	=& \underbrace{\tpi_n (\theta_{n_0} - \theta^*)}_{\text{Initial Error}}
		- \underbrace{\sum\limits_{k=1}^{n} \tpi_n \tpi_k^{-1} (\ttheta_{k} - \ttheta_{k-1})}_{\text{Drift Error}}	+ \underbrace{\sum\limits_{k=1}^{n} \gamma_k \tpi_n \tpi_k^{-1} \Delta \tM_k}_{\text{Sampling Error}},\label{eq:reg-decomp}
\end{align}
where $\tpi_n:=\prod_{k=1}^{n}\left(I - \gamma_k (\bar A_k+ \lambda_k I)\right)$. Under the assumption that we have no control over the smallest eigenvalue of $\bar A_k$, we can only upper bound the initial error by $\exp(-\sum_{k = 1}^n \gamma_k\lambda_k)$. Therefore, in order that the initial error go to zero we must have that $\sum_{k = 1}^n \gamma_k\lambda_k\rightarrow \infty$ as $n\rightarrow \infty$. Taking a step size of the form $\gamma_n = O(n^{-\alpha})$ therefore forces $\lambda_n = \Omega(n^{-(1-\alpha)})$. However, examining the drift
\begin{align*}
\ttheta_{n-1} - \ttheta_n =& \xi_n (A_{n-1}+(n-1)\lambda_{n-1}I)^{-1}x_n  - (x_n\tr(\ttheta_n -\theta^*))(A_{n-1}+(n-1)\lambda_{n-1}I)^{-1}x_n\\
	&\qquad\qquad\qquad\qquad\qquad\qquad+ ((n-1)\lambda_{n-1} - n\lambda{n})(A_{n-1}+(n-1)\lambda_{n-1}I)^{-1}\ttheta_{n}.
\end{align*}
So when $\lambda_n = \Omega(n^{-(1-\alpha)})$, then we find that $\ttheta_{n-1} - \ttheta_n = \Omega(n^{-1})$, whenever $\alpha\in(0,1)$. This, when plugged into \eqref{eq:reg-decomp} results in only a constant bound on the error (note, $\alpha$ must be chosen in $(1/2,1)$ to ensure (A1) holds).
Unlike in the setting of \cite{tarres2011online}, we do not assume that the data arrive from a distribution, and hence the bias error is difficult to control.

%%%%%%%%%%%%%%%%%%%%%%%%%%%%%%%%%%%%%%%%%%%%%%%%%%%%%%%%%%%%%%
%%%%%%%%%%%%%%%%%%%%%%%%%%%%%%%%%%%%%%%%%%%%%%%%%%%%%%%%%%%%%%
\section{Numerical Experiments}
\label{sec:fRLS-GD-bandits}
\paragraph{Background for LinUCB} In this section the action sets $\cD_n\subset \cD$ are finite, but possibly varying. A popular algorithm for such settings is the LinUCB algorithm. This algorithm calculates UCBs for the mean reward obtained by choosing each individual feature in $\cD_n$ as follows:
\begin{align*}
\forall x\in\cD_n, \quad UCB(x):= x\tr \hat\theta_n + \kappa\sqrt{ x\tr A_n^{-1} x},
\end{align*}
where $\kappa$ is a parameter set by the agent that can be understood to be controlling the rate of exploration the algorithm performs.
Having calculated the UCBs for all available features the agent then chooses the feature with the highest UCB.
LinUCB needs to compute online the inverse of the matrix $A_n^{-1}$ in order to compute the UCBs for each iteration of the algorithm, and so we propose improving the complexity by using an SGD scheme to approximate the UCBs.
Since LinUCB cannot guarantee strong convexity of the regression problem, we investigate experimentally applying the regularised fRLS-GD in place of RLS solutions. 

\paragraph{Tracking the UCBs}
While we can track the regularised estimates $\ttheta_n$ using fRLS-GD as given above, to track the UCBs we derive the analogous update rule for each feature $x\in\cD_n$:
\begin{align}
 \label{eq:roarlsa-update-variation}
\phi_n = \phi_{n-1} + \gamma_n( (n^{-1} x - ((\phi_{n-1})\tr x_{(i_n)})x_{(i_n)})),
\end{align}
 where $i_n\sim\cU({1,\dots,n})$. The UCB value corresponding to feature $x$ is then set as follows:
\begin{align*}
UCB(x):= {x}\tr \theta_n + \kappa\sqrt{ {x}\tr \phi_{n} }.
\end{align*}
If the action sets $\cD_n$ were fixed (say $\cD_1$), then we take one step according to \eqref{eq:roarlsa-update-variation} for each arm $x \in \cD_1$ in each iteration $n$ of LinUCB. The computational cost of this LinUCB variant is of order $O(|\cD_1 |dn)$, as opposed to the $O(d^2n)$ incurred by the vanilla LinUCB algorithm that directly calculates $A_n^{-1}$.
This variant would give good computational gains when $|\cD_1 |\ll d$.
If the action sets change with time, then one can perform a batch update, i.e., run $T$ steps according to \eqref{eq:roarlsa-update-variation} for each feature $x \in \cD_n$ in iteration $n$ of LinUCB.
This would incur a computational complexity of order $O(KTdn)$, where $K$ is an upper bound on $|\cD_n|$ for all $n$, and 
result in good computational gains when $KT\ll d$.

%%%%%%%%%%%%%%%%%%%%%%%%%%%%%%%%%%%%%%%%%%%%%%%%%%%%%%%%%%%%%%
\algblock{PEval}{EndPEval}
\algnewcommand\algorithmicPEval{\textbf{\em Approximate RLS}}
 \algnewcommand\algorithmicendPEval{}
\algrenewtext{PEval}[1]{\algorithmicPEval\ #1}
\algrenewtext{EndPEval}{\algorithmicendPEval}

\algblock{PImp}{EndPImp}
\algnewcommand\algorithmicPImp{\textbf{\em UCB computation}}
 \algnewcommand\algorithmicendPImp{}
\algrenewtext{PImp}[1]{\algorithmicPImp\ #1}
\algrenewtext{EndPImp}{\algorithmicendPImp}

\algtext*{EndPEval}
\algtext*{EndPImp}

\begin{algorithm}  
\caption{fLinUCB-GD}
\label{alg:LinUCBwithBGD}
\begin{algorithmic}
\State {\bfseries Initialisation:} Set $\theta_0$, $\gamma_k$ - the step-size sequence.
  \For{$n = 1,2,\ldots$}
  \PEval
% 	\State \bf{\emph{Approximate RLS}}
	\State Observe article features $x_{n}^{(1)},\ldots,x_{n}^{(K)}$
    \State Approximate $\hat\theta_n$ using fRLS-GD iterate $\theta_n$ \eqref{eq:fRLS-GD-update}
  \EndPEval
  \PImp        
	\For{$k=1,\ldots,K$}
        \State Estimate confidence parameter $\phi_{n}^{(k)}$ using \eqref{eq:roarlsa-update-variation}
        \State Set $\textrm{UCB}(x_{n}^{(k)}) := \theta_{n}\tr x_{n}^{(k)} + \kappa \sqrt{ {x_{n}^{(k)}}\tr \phi_{n}^{(k)} }$
        \EndFor        
  \EndPImp     
  \State Choose article $\argmax_{k=1,\ldots,K} \textrm{UCB}(x_{n}^{(k)})$ and observe the reward.
  \EndFor
\end{algorithmic}
\end{algorithm}

%%%%%%%%%%%%%%%%%%%%%%%%%%%%%%%%%%%%%%%%%%%%%%%%%%%%%%%%%%%%%%
\label{sec:experiments}
\paragraph{Simulation Setup.}
We perform experiments on a news article recommendation platform provided for the ICML exploration and exploitation challenge (\cite{li12}).
This platform is based on the user click log dataset from the Yahoo! front page, provided under the Webscope program (\cite{webscope}). 
An algorithm for this platform is required to repeatedly select a news article from a pool of articles and show them to users.
Each article-user pair is described in the dataset by a feature vector, which the algorithm can use to make its decisions.

\begin{figure}
    \centering
     \includegraphics[width=2.3in]{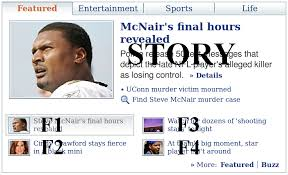}
    \caption{The {\em Featured} tab in Yahoo! Today module \citep{li2010contextual}}
\end{figure}
 
 \begin{figure}
    \centering
\tabl{c}{\scalebox{0.9}{\begin{tikzpicture}
\begin{axis}[
ybar={2pt},
legend style={at={(0.5,-0.2)},anchor=north,legend columns=-1},
% legend pos=outer north east,
% legend columns=2,
legend image code/.code={\path[fill=white,white] (-2mm,-2mm) rectangle
(-3mm,2mm); \path[fill=white,white] (-2mm,-2mm) rectangle (2mm,-3mm); \draw
(-2mm,-2mm) rectangle (2mm,2mm);},
ylabel={\bf runtime (ms)},
xlabel={},
symbolic x coords={0, 1, 2, 3},
xmin={0},
xmax={3},
xtick=data,
ytick align=outside,
xticklabels={{\bf Day-2,\bf Day-4}},
xticklabel style={align=center},
bar width=16pt,
nodes near coords,
    every node near coord/.append style={font=\tiny},
   nodes near coords align={vertical},
   grid,
grid style={gray!30},
width=13cm,
height=7cm,
]
\addplot   coordinates {  (1,1367018) (2,1717997) }; 
\addlegendentry{LinUCB}
\addplot coordinates {  (1,4933) (2,6474) }; 
\addlegendentry{fLinUCB-GD}
\addplot coordinates {  (1,81818) (2,107386) }; 
\addlegendentry{fLinUCB-SVRG}
\addplot coordinates {  (1,44504) (2,55630) }; 
\addlegendentry{fLinUCB-SAG}
\end{axis}
\end{tikzpicture}}\\[1ex]}
\caption{Runtimes (in ms) on two days of the dataset for LinUCB and its SGD variants}
\label{fig:runtimes}
\end{figure}
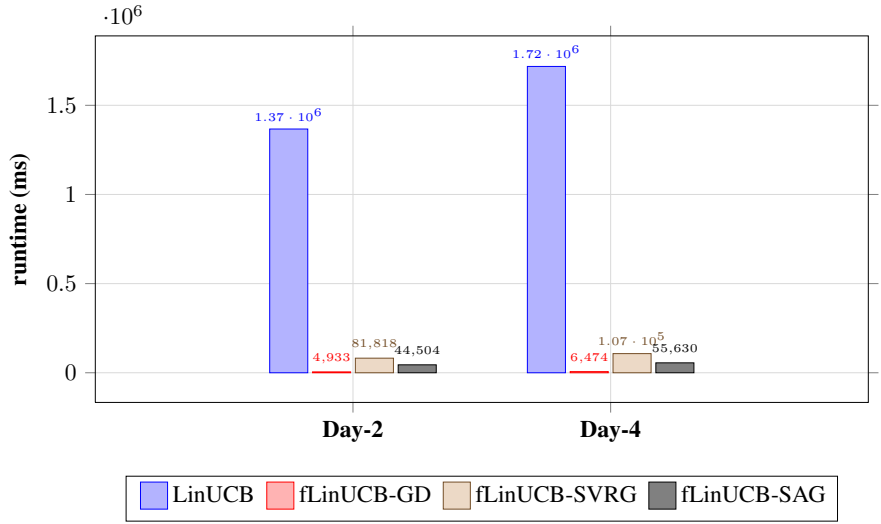

We implement the LinUCB algorithm (popular for this setting) as well as three SGD variants. The first SGD variant is based on fRLS-GD, while the other two variants are based on two recent approaches for accelerating the convergence of SGD-type schemes. We describe these below.
\begin{description}
 \item[fLinUCB-GD.] This is described in Algorithm \ref{alg:LinUCBwithBGD} and uses fRLS-GD in place of RLS.
 \item[fLinUCB-SVRG.] This is similar to the above algorithm, except that the SGD scheme used is derived from \cite{johnson2013accelerating}.
The first scheme is derived from \cite{johnson2013accelerating} and updates the parameter as follows:
  Let $f_{i,n}(\theta):=\frac{1}{2} (y_i - \theta\tr x_i)^2 + \lambda_n \l\theta\r^2$, $F_n(\theta)= \frac1{n}\sum\limits_{i=1}^{n-1} f_{i,n}(\theta)$ and $\bar\theta_{n} = \sum\limits_{i=1}^{n-1} \theta_i$. Then,
\begin{align}
 \label{eq:svrg-update}
 \theta_n = \theta_{n-1} - \gamma_n( f'_{i_n}(\theta_{n-1}) - f'_{i_n}(\bar\theta_n) + F'_n(\bar\theta_n)),
\end{align}
where $i_n$ is picked uniformly at random in $\{1,\ldots,n\}$.
 \item[fLinUCB-SAG.] This is a variant that uses the SGD scheme proposed by \cite{roux2012stochastic}. The  updates here are according to
\begin{equation}
\label{eq:SAG}
\theta_{n} = \theta_{n-1} - \dfrac{\gamma_n}{n}\sum_{i=1}^n y_{n,i}, \quad \text{ where } \quad y_{n,i} = \begin{cases}
f_i'(\theta_{n-1}) & \textrm{if $i = i_n$,}\\
y_{n-1,i} & \textrm{otherwise.}
\end{cases}
\end{equation}
where $i_n$ is picked uniformly at random in $\{1,\ldots,n\}$.
\end{description}

\begin{remark}
The last two SGD schemes presented above are shown to converge at a geometric rate for a single-batch training problem, while SGD can converge only at $O(1/n)$ rate. However, this rate acceleration comes at an additional computational cost in comparison to regular SGD. 
Moreover, in a drifting least squares regression setting that we consider in this paper, both these variants would still suffer from a drift error as discussed in Section \ref{sec:fRLS-GD} of the main paper and hence, obtaining a sub-linear rate of convergence is challenging even for these schemes.
\end{remark}

%%%%%%%%%%%%%%%%% Tracking error plots %%%%%%%%%%%%%%%%%%%%%%%%%%%
 \begin{figure}
    \centering
    \begin{tabular}{cc}
        %%%%%% normdiff plot
        \subfigure[Tracking error: fRLS-GD]
        {
	 \hspace{-2em}\tabl{c}{\scalebox{0.85}{\begin{tikzpicture}
	    \begin{axis}[xlabel={iteration $n$ of flinUCB-GD},ylabel={$\l\theta_{n} - \ttheta_n\r$},ytick pos=left,xtick pos=left,grid,grid style={gray!30},ymax=1,restrict x to domain=0:50000]
	    \addplot+[blue] table[x index=0,y index=1,col sep=space] {results/onlinegd_subsampled.log};
	    \addlegendentry{fRLS-GD}% y index+1 since humans count from 1
	    \end{axis}
	    \end{tikzpicture}}\\[1ex]}
           \label{fig:normdiff1}	    
        }
&
        \subfigure[Tracking error: fRLS-SVRG]
        {
	    \hspace{-2em}\tabl{c}{\scalebox{0.85}{\begin{tikzpicture}
	 \begin{axis}[xlabel={iteration $n$ of flinUCB-SVRG},ylabel={$\l\theta_{n} - \ttheta_n\r$},ytick pos=left,xtick pos=left,grid,grid style={gray!30},ymax=1,restrict x to domain=0:50000]
	    \addplot+[red] table[x index=0,y index=1,col sep=space] {results/svrg_subsampled.log};
	    \addlegendentry{fRLS-SVRG}% y index+1 since humans count from 1
	    \end{axis}
	    \end{tikzpicture}}\\[1ex]}
           \label{fig:normdiff2}	    
        }
\end{tabular}
\begin{tabular}{c}
         \subfigure[Tracking error: fRLS-SAG]
        {
	  \hspace{-2em}\tabl{c}{\scalebox{0.85}{\begin{tikzpicture}
	   \begin{axis}[xlabel={iteration $n$ of flinUCB-SAG},ylabel={$\l\theta_{n} - \ttheta_n\r$},ytick pos=left,xtick pos=left,grid,grid style={gray!30},ymax=1,restrict x to domain=0:50000]
	    \addplot+[green] table[x index=0,y index=1,col sep=space] {results/sag_subsampled.log};
	    \addlegendentry{fRLS-SAG}% y index+1 since humans count from 1
	    \end{axis}
	    \end{tikzpicture}}\\[1ex]}
           \label{fig:normdiff3}	    
        }
\end{tabular}
\caption{Performance evaluation of fast LinUCB variants using tracking error} 
\end{figure}
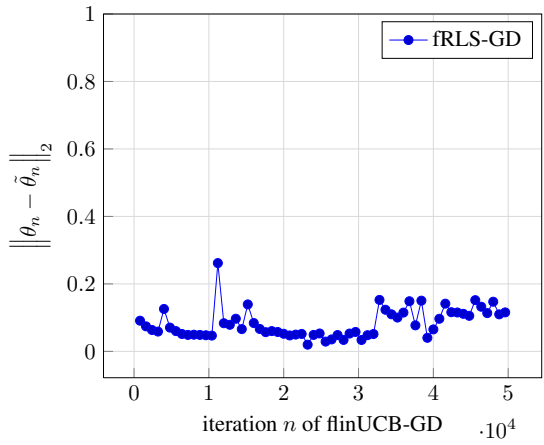
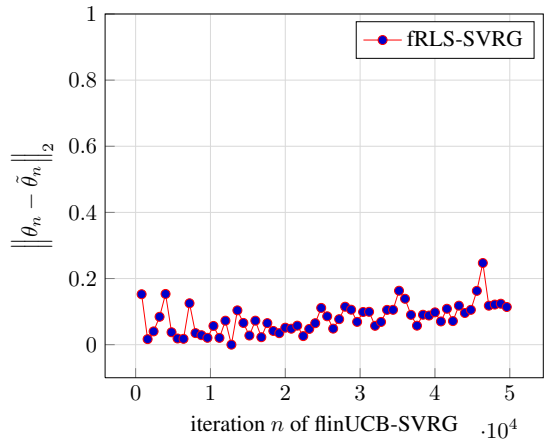
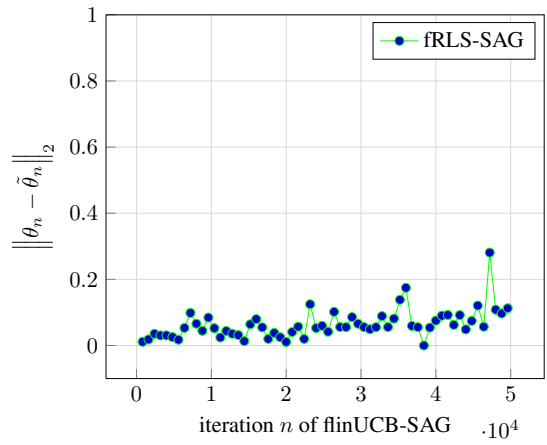

We set the various parameters of the problem as well as SGD algorithms as follows:\\
\begin{table}[h]
\begin{tabular}{|c|c|}
\hline
\multirow{2}{*}{\textbf{Algorithm}} & \multirow{2}{*}{\textbf{Parameters}} \\ 
 &  \\ \hline
\textbf{fLinUCB-GD} & Regularisation parameter $\lambda_n = \dfrac{1}{n^{1-\alpha}}$, $\alpha=0.6$, stepsize $\gamma_n = \dfrac{1}{100+n}$ \\ \hline
\textbf{fLinUCB-SVRG} &  Regularisation parameter $\lambda_n = \dfrac{1}{n}$, stepsize $\gamma_n = 0.0005$ \\ \hline
\textbf{fLinUCB-SAG} &  Regularisation parameter $\lambda_n = \dfrac{1}{n}$, stepsize $\gamma_n = 0.005$ \\ \hline
\end{tabular}
\end{table}

\paragraph{Results} We use tracking error and runtimes as performance metrics for comparing the algorithms. The tracking error is the difference in $\ell^2$ norm between the SGD iterate $\theta_n$ and RLS solution $\ttheta_n$, at each instant $n$ of the SGD variant of LinUCB. 

Figs. \ref{fig:normdiff1}--\ref{fig:normdiff3} present the tracking error with day $2$'s data file as input for fRLS-GD, SVRG and SAG variants of LinUCB, respectively. It is evident that all the SGD schemes track the corresponding RLS solutions consistently. 
Fig. \ref{fig:runtimes} report the runtimes observed on two different data files corresponding to days $2$ and $4$ in October, $2009$ (see \cite{webscope}) of the dataset. 
It is evident that the SGD schemes result in significant computational gains in comparison to classic RLS solvers (e.g. Sherman-Morrison lemma). 

Finally, we observed that the SGD variants under best configurations achieved $75\%$ of the regular LinUCB CTR score. CTR score is the ratio of the number of clicks an algorithm gets to the total number of iterations it completes, multiplied by $10000$. Considering that the dataset contains very sparse features and also the fact that the rewards are binary, with a reward of $1$ occurring rarely, we believe LinUCB has not seen enough data to have converged UCB values and hence the observed loss in CTR may not be conclusive.

%%%%%%%%%%%%%%%%%%%%%%%%%%%%%%%%%%%%%%%%%%%%%%%%%%%%%%%%%%%%%%
%%%%%%%%%%%%%%%%%%%%%%%%%%%%%%%%%%%%%%%%%%%%%%%%%%%%%%%%%%%%%%
\section{Conclusions}
\label{sec:conclusions}
We analysed online SGD schemes for the problem of drifting least squares regression problems in the context of a higher level algorithm.
In particular, when the higher level algorithm can guarantee strong convexity in the data, we provided error bounds both in expectation and in high probability.
Further, we derived an SGD variant of PEGE linear bandit algorithm with a speed up of $O(d)$ at the cost of only logarithmic factors in the regret.
For the non-strongly convex setting, we studied an adaptively regularised SGD scheme by combining it with LinUCB.
The empirical results of this algorithm on a large-scale news recommendation application are encouraging.
However a theoretical analysis of the adaptively regularised SGD scheme remains challenging, and is an interesting direction for future work.

\paragraph{\bf Acknowledgments}
The first author was gratefully supported by the EPSRC project, Autonomous Intelligent Systems EP/I011587.
The second and third authors would like to thank the European Community's Seventh Framework Programme (FP$7/2007-2013$) under grant agreement n$^o$ $270327$ for funding the research leading to these results.

\section*{Appendix}

\appendix

%%%%%%%%%%%%%%%%%%%%%%%%%%%%%%%%%%%%%%%%%%%%%%%%%%%%%%%%%%%%%%
\section{Proof of Theorem \ref{cor:fOLS-GD}}
\label{sec:analysis}
Let $z_n:= \theta_n - \hat\theta_n$ denote the approximation error. Throughout this proof, we shall assume that $n\ge n_0$, in accordance with assumption (A4). The proof involves the following steps:
  \begin{description}
   \item[Step 1] Proposition \ref{thm:online-high-prob} bounds the deviation of $z_n$ from its mean in high probability;
   \item[Step 2] Proposition \ref{thm:online-exp} bounds the mean of $z_n$ itself;
   \item[Step 3] the final step is to combine the above two propositions, with the step-sizes $\gamma_n$ chosen as $c/(4(c+n))$.
  \end{description}
In the following, we describe each of the individual steps above in detail.

\subsection*{\underline{Step 1: High-probability bound}}
In the following proposition, we  bound the deviation in high probability of the approximation error, $z_n:= \theta_n - \hat\theta_n$, from its mean.  The proof technique is similar to that used by \cite{frikha2012concentration}. However, our analysis is much simpler, we make all the constants explicit for the problem at hand, and we deal with the extra error incurred as a result of the drifting target $\hat\theta_n$. 
\begin{proposition}
\label{thm:online-high-prob}
Let $z_n:=\theta_n - \hat\theta_n$. Then,
under (A1)-(A3), for all $n\geq d$, we have
\begin{align*}
P(   \l z_n \r - \E \l z_n \r \ge \epsilon ) \le \e\left(- \epsilon^2\left( 2 \sum_{i=n_0+1}^{n} L^2_i\right)^{-1}\right),
\end{align*}
where $L_i^2 := \gamma_i^2 \prod_{j=i}^{n-1} (1 - 2 \mu \gamma_{j+1}\left(1-\gamma_{j+1}\right))$.
\end{proposition}
\begin{proof}
% \paragraph{Step 1} 
Let $\cH_i$ be the sigma field generated by the random variables $i_0,\ldots,i_n$ and $\xi_1,\dots,\xi_n$. We decompose $\l z_n \r^2 - E \l z_n \r^2$ into a sum of martingale differences as follows:
\begin{align}
 \l z_n \r - \E \l z_n \r = & \sum\limits_{i=n_0 + 1}^{n} \E [\l z_n\r \left| \cH_i \right.]  - \E [\l z_n\r \left| \cH_{i-1} \right.] \nonumber\\
= & \sum\limits_{i=n_0 + 1}^{n} \E [\l z_n\r \left| \theta_i \right.]  - \E [\E(\l z_n\r\left| \theta_i \right.) \left| \cH_{i-1} \right.] 
= \sum\limits_{i=n_0 + 1}^{n} D_i, \label{eq:prob-equivalence-appendix}
\end{align}
where $D_i \stackrel{\triangle}{=} g_i - \E [ g_i \left| \cH_{i-1} \right.]$ and $g_i =  \E [\l z_n\r \left| \theta_i \right.]$.

We now establish that the functions $g_i$ are Lipschitz continuous in the noise $\xi_i$, with Lipschitz constants $L_i$. 
We are interested in measuring the difference in the iterate $\theta_n$ at instant $n$, while starting from two different initial values at instant $i$.
To do this, let $\Theta_n^i(\theta)$ denote the $n^{th}$ iterate, $\theta_n$, given that at instant $i$, we set it to $\theta$ (i.e., $\theta_i = \theta$). 
Then from the equalities
\begin{align*}
\Theta_n^i(\theta) - \Theta_n^i(\theta')
	= \left(I - \gamma_n x_{i_n}x_{i_n}^T\right)\left[\Theta_{n-1}^i(\theta) - \Theta_{n-1}^i(\theta')\right]
\end{align*}
and
\begin{align*}
\left(I - \gamma_n x_{i_n}x_{i_n}^T\right)^T\left(I - \gamma_n x_{i_n}x_{i_n}^T\right)
 = \left(I -2 \gamma_n (1-\|x_{i_n}\|_2^2\gamma_n)x_{i_n}x_{i_n}^T\right),
\end{align*}
using Jensen's inequality, and Cauchy-Schwarz, we can deduce that
\begin{align*}
E\left[ \| \Theta_n^i(\theta) - \Theta_n^i(\theta')\|_2 \mid \Theta_{n-1}^i(\theta), \Theta_{n-1}^i(\theta')\right]
	\le\left[  \| I- 2\gamma_n(1-\gamma_n)\bar{A}_{n-1}\|_2\| \Theta_{n-1}^i(\theta) - \Theta_{n-1}^i(\theta')  \|_2^2 \right]^{1/2}
\end{align*}
Unrolling this iteration, and using the Tower property of conditional expectations, and assumption (A3), we find that
\begin{align*}
\E\left[ \l \Theta_n^i(\theta) - \Theta_n^i(\theta') \r \right]
	\le \l \theta-\theta' \r  \prod_{j = i+1}^{n}(1+2\mu\gamma_j (1-\gamma_j) )^{\frac{1}{2}}.
\end{align*}
Finally we have
\begin{align*}
&\left| \E\left[ \l \theta_n - \hat\theta_n \r \left| \theta_{i-1}, \xi_{i_i} = \xi\right.\right]\right.
	\left. - \E\left[ \l \theta_n - \hat\theta_n \r \left| \theta_{i-1}, \xi_{i_i} = \xi' \right.\right] \right| \\
 	& \quad \le  \E\left[ \l \Theta^i_n\left(\theta\right) - \Theta^i_n\left(\theta'\right) \r\right]
	 \quad \le  \left[\gamma_i  \prod_{j = i+1}^{n}(1+2\mu\gamma_j(1-\gamma_j))^{\frac{1}{2}}\right]\left| \xi - \xi'\right| 
	 =  L_i | \xi - \xi'|.
\end{align*}

The last step of the proof is to invoke a concentration bound for sum of martingale differences $D_i$: First note that
  \begin{align*}
  P(   \l z_n \r - \E \l z_n \r \ge \epsilon ) = & P\left(  \sum\limits_{i=1}^{n} D_i  \ge \epsilon \right) 
  \le \e(-\lambda \epsilon) \E\left(\e\bigg(\lambda \sum\limits_{i=1}^{n} D_i\bigg)\right) \\
  = & \e(-\lambda \epsilon) \E\left(\e\bigg(\lambda \sum\limits_{i=1}^{n-1} D_i\bigg) \E\bigg(\e(\lambda D_n \left| \cH_{n-1}\right.)\bigg)\right).  
 \end{align*}
The first equality above follows from \eqref{eq:prob-equivalence-appendix}, while the inequality follows from Markov inequality.
Since $\xi_i$ are bounded by (A2), we have the following property that holds for every $1$-Lipschitz function $g$, we have
\begin{align*}
 \E \left( \e(\lambda g(\xi_1))\right) \le \e\left( \dfrac{\lambda^2}{2} \right).
\end{align*}
Noting that $g_i$ is Lipschitz with constant $L_i$, we apply the above inequality to obtain
\begin{align*}
 \E\left(\e(\lambda D_n \left| \cH_{n-1}\right.)\right) \le \e\left(\dfrac{\lambda^2 L^2_n}{2}\right),
\end{align*}
and so
  \begin{align*} 
   P(   \l z_n \r - \E \l z_n \r \ge \epsilon ) \le \e(-\lambda \epsilon) \e\bigg(\dfrac{\alpha \lambda^2}{2} \sum\limits_{i=n_0 + 1}^{n} L^2_i \bigg)
\end{align*}
The claim follows by optimizing over $\lambda$ in the above.
\end{proof}

%%%%%%%%%%%%%%%%%%%%%%%%%%%%%%%%%%%%%%%%%%%%%%%%%%%%%%%%%%%%%%%%%%%%%%%%%%%%%%%%%%%%%%%%%%%%%%%%%%%%%%%%%%%%%%%%%%%%%%%%%%%%%%%%%%%%%%%%%%%%%%%%%%%
%%%%%%%%%%%%%%%%%%%%%%%%%%%%%%%%%%%%%%%%%%%%%%%%%%%%%%%%%%%%%%%%%%%%%%%%%%%%%%%%%%%%%%%%%%%%%%%%%%%%%%%%%%%%%%%%%%%%%%%%%%%%%%%%%

\subsection*{\underline{Step 2: Bound in expectation}}
The following proposition bounds the expected value of the approximation error $z_n$. The proof differs from earlier works on SGD techniques, as it involves a certain drift term that requires special attention.
\begin{proposition}
Let $z_n:=\theta_n - \hat\theta_n$. Then,
under (A1)-(A3), for all $n\geq n_0$, we have
\label{thm:online-exp}
\begin{align*}
\E& \l z_n \r \le \quad \underbrace{\e(-\mu [\Gamma_n - \Gamma_{n_0}]) \l z_{n_0} \r}_{\text{initial error}}
	 + \underbrace{\left(\sum_{k=n_0+1}^{n} h(k)\gamma_{k}^2\e (-2\mu(\Gamma_n - \Gamma_{k}))\right)^{1/2}}_{\text{sampling error}}\\
	 &+ \underbrace{\left(\sum_{k=n_0+1}^{n}\e\left(-2\mu(\Gamma_n - \Gamma_k)\right) \dfrac{1}{\mu^2(k-1)^{2}}  \right)^{1/2}
	 + \sum_{k=n_0+1}^{n}\e\left(-\mu(\Gamma_n - \Gamma_k)\right) \l \hat\theta_k -\theta^* \r \dfrac{1}{\mu (k-1)}}_{\text{drift error}}.
\end{align*}
where $\Gamma_k:=\sum_{i=1}^k \gamma_i$, $h(k) := 2\left[\sigma_\xi^2 + 2(\l z_0\r +\Gamma_k)^2\right]$, with $\sigma_\xi := \E_\xi[\xi^2]<\infty$ denoting the variance of the noise.
\end{proposition}

\begin{proof}
As above, let $f_{n}(\theta) := \frac{1}{2} (\xi_{i_n} - (\theta-\theta^*)\tr x_{i_n})^2$, $F_n(\theta) := \E_{i_n}[ f_{n}(\theta)\mid \cH_n]$, and $\Delta M_{n+1}$ be the associated martingale difference sequence, $\Delta M_{n+1}(\theta) := F_n'(\theta) - f'_{n}(\theta)$.
We find a recursion for $z_n = \theta_n - \hat\theta_n$ by extracting a martingale difference from the process:
\begin{align*}
z_n =&  \theta_n - \hat\theta_{n-1} + \hat\theta_{n-1} - \hat\theta_{n}
    = z_{n-1} - \gamma_n\left(F'_{n}(\theta_{n-1}) - \Delta M_n\right) +  (\hat\theta_{n-1} - \hat\theta_{n})\\
	= & z_{n-1} -  \gamma_n\left(\bar A_n z_{n-1} - \Delta M_n\right) +  (\hat\theta_{n-1} - \hat\theta_{n})\\
	= &  \left(1 - \gamma_n\bar A_n\right)z_{n-1} + \gamma_n\Delta M_{n} +  (\hat\theta_{n-1} - \hat\theta_{n})\\
%	=&  \prod_{k=1}^{n}\left(I - \gamma_k\bar A_k\right)z_0 - \sum_{k=1}^{n}\left[\prod_{j=k+1}^{n}\left(I - \gamma_j\bar A_j \right)\right](\gamma_k\Delta M_k +  (\hat\theta_{k} - \hat\theta_{k-1}) )\\
	=&  \Pi_n \Pi_{n_0}^{-1}  z_{n_0} - \sum\limits_{k=n_0 + 1}^{n} \Pi_n \Pi_k^{-1} (\hat\theta_{k} - \hat\theta_{k-1}) + \sum\limits_{k=n_0 + 1}^{n} \gamma_k \Pi_n \Pi_k^{-1} \Delta M_k,
\end{align*}
where $\Pi_n:=\prod_{k=1}^{n}\left(I - \gamma_k \bar A_k\right)$. The third equality uses the fact that $F'_n(\hat\theta_{n-1}) = 0$,
By Jensen's inequality, we obtain
\begin{align}
\E \l z_n \r \le  \l  \Pi_n \Pi_{n_0}^{-1}  z_{n_0} \r + \E \l \sum_{k=n_0 + 1}^{n}\Pi_n \Pi_k^{-1} (\hat\theta_{k+1} - \hat\theta_{k})\r + \left( \sum_{k=n_0 + 1}^{n}\gamma_k^2 \Pi_n\Pi_k^{-1}\E\l\Delta M_k\r^2\right)^{1/2} \label{eq:thm2-it-appendix}
\end{align}
Note that
\begin{align}
\Pi_n\Pi_k^{-1}  \leq \prod_{j=k+1}^{n}(1-\mu\gamma_j)\leq \e\left(\log\left(\prod_{j=k+1}^{n}(1-\mu\gamma_j) \right)\right)
	\le  \e\left(-\mu(\Gamma_n - \Gamma_k)\right),\label{eq:thm2-pi}
\end{align}
where $\Gamma_n: = \sum_{i = 1}^n \gamma_i$.
We now bound each of the terms in \eqref{eq:thm2-it-appendix} as follows:
\begin{description}
 \item[First term] From \eqref{eq:thm2-pi} we see that $\l \Pi_n \Pi_{n_0}^{-1} z_{n_0} \r \le \e (-\mu[\Gamma_n - \Gamma_{n_0}])\l z_{n_0}\r$.
 \item[Second term]
 Since $\hat\theta_n$ and $\hat\theta_{n-1}$ are solutions to the least squares problems at instants $n$ and $n-1$, respectively, we have
$$\sum_{i=1}^{n} (y_i - x_i\tr \hat\theta_n)x_i = 0 = \sum_{i=1}^{n-1} (y_i - x_i\tr \hat\theta_{n-1})x_i.$$
Simplifying the above, we obtain
$$
A_{n-1}(\hat\theta_{n-1} - \hat\theta_n) + (y_n-x_n\tr\hat\theta_n)x_n =0
$$
$$
\Leftrightarrow \hat\theta_{n-1} - \hat\theta_n = A_{n-1}^{-1}(x_n\tr\hat\theta_n - (x_n\tr\theta^* + \xi_n))x_n
$$
$$
\Leftrightarrow \hat\theta_{n-1} - \hat\theta_n = \xi_n A_{n-1}^{-1}x_n  - (x_n\tr(\hat\theta_n -\theta^*))A_{n-1}^{-1}x_n.
$$
Therefore, we have 
\begin{align}
 \sum_{k=1}^{n}\Pi_n \Pi_k^{-1} (\hat\theta_{k} - \hat\theta_{k-1})
  = \sum_{k=1}^{n}\Pi_n \Pi_k^{-1} A_{k-1}^{-1}x_n \xi_n
  	- \sum_{k=1}^{n}\Pi_n \Pi_k^{-1} (x_n\tr(\hat\theta_k -\theta^*))A_{k-1}^{-1}x_n,
\end{align}
So once again applying Jensen's inequality, using that the noise $\xi_n$ is zero mean and bounded by $1$, and assumptions (A1) and (A3), we have
\begin{align*}
\E\l \sum_{k=1}^{n}\Pi_n \Pi_k^{-1} (\hat\theta_{k} - \hat\theta_{k-1}) \r
 \le \left(\sum_{k=n_0 +1}^{n}\left(\Pi_n \Pi_k^{-1} \frac{1}{\mu (k-1)}\right)^{2}  \right)^{\frac{1}{2}}
	+ \sum_{k=n_0 + 1}^{n}\Pi_n \Pi_k^{-1} \l \hat\theta_k -\theta^* \r \frac{1}{\mu (k-1)},
\end{align*}
 
 \item[Last term] The martingale difference (last term in \eqref{eq:thm2-it-appendix}) is bounded as below:
\begin{align*}
\E [\l \Delta M_n \r^2] 
	\le 2\left(\E\langle f'_{i_n}(\theta_{n-1}), f'_{i_n}(\theta_{n-1})\rangle + \E\langle F_n'(\theta_{n-1}), F_n'(\theta_{n-1})\rangle\right)
\end{align*}
Using (A1) and (A2), a simple calculation shows that
\begin{align*}
\E\langle f'_{i_n}(\theta_{n-1}), f'_{i_n}(\theta_{n-1})\rangle,\ \E \langle F_n'(\theta_{n-1}), F_n'(\theta_{n-1})\rangle  \le n^\frac{1}{2} + (1+2n^\frac{1}{2}(1+\lambda_n^2))\E\l z_n\r)
\end{align*}
Now
\begin{align*}
\E\l z_n\r =&\E \l \left[\prod_{k=n_0 + 1}^n(I - \gamma_{k}x_{i_k}x_{i_k}\tr )\right]z_0
	+ \sum_{k = 1}^n\gamma_k \left[\prod_{j=n_0 + k+1}^n(I - \gamma_{j}x_{i_j}x_{i_j}\tr)\right](\xi_k x_{i_k} +\lambda_k \theta^*)\r\\
	\le& \e (-\mu[\Gamma_n -\Gamma_{n_0}])\l z_{n_0} \r + \left(\sum_{k=n_0 + 1}^n \gamma_k^2\e(-2(\Gamma_n - \Gamma_k))\right)^\frac{1}{2}n^\frac{1}{2} + \Gamma_n\l\theta^*\r=:g(n).
\end{align*}
and so $\E [\l \Delta M_n \r^2] \le h(n)$ where $h(n) = n^\frac{1}{2} + (1+2n^\frac{1}{2}(1+\lambda_n^2))g(n).$
\end{description}

Putting it all together, \eqref{eq:thm2-it-appendix} simplifies to the following form:
\begin{align*}
\E \l z_n\r \le & \l z_{n_0}\r \e (-\mu[\Gamma_n-\Gamma_{n_0}]) 	
	 + \left(\sum_{k=n_0 + 1}^{n}\e\left(-2\mu(\Gamma_n - \Gamma_k)\right) \frac{1}{\mu^2 (k-1)^2}  \right)^{\frac{1}{2}}\\
	 &+ \sum_{k=n_0 + 1}^{n}\e\left(-\mu(\Gamma_n - \Gamma_k)\right) \l \hat\theta_k -\theta^* \r \frac{1}{\mu (k-1)}
	 + \left(\sum_{k=n_0 + 1}^{n} h(k)\gamma_{k}^2\e (-2\mu(\Gamma_n - \Gamma_{k}))\right)^{1/2}.
\end{align*}
\end{proof}

%%%%%%%%%%%%%%%%%%%%%%%%%%%%%%%%%%%%%%%%%%%%%%%%%%%%%%%%%%%%%%%%%%%%%%%%%%%%%%%%%%%%%%%%%%%%%%%%%%%%%%%%%%%%%%%%%%%%%%%%%%%%%%%%%%%%%%%%%%%%%%%%%%%
%%%%%%%%%%%%%%%%%%%%%%%%%%%%%%%%%%%%%%%%%%%%%%%%%%%%%%%%%%%%%%%%%%%%%%%%%%%%%%%%%%%%%%%%%%%%%%%%%%%%%%%%%%%%%%%%%%%%%%%%%%%%%%%%%%%%%%%%%%%%%%%%%%%
%%%%%%%%%%%%%%%%%%%%%%%%%%%%%%%%%%%%%%%%%%%%%%%%%%%%%%%%%%%%%%%%%%%%%%%%%%%%%%%%%%%%%%%%%%%%%%%%%%%%%%%%%%%%%%%%%%%%%%%%%%%%%%%%%

%%%%%%%%%%%%%%%%%%%%%%%%%%%%%%%%%%%%%%%%%%%%%%%%%%%%%%%%%%%%%%%%%%%%%%%%%%%%%%%%%%%%%%%%%%%%%%%%%%%%%%%%%%%
\subsection*{\underline{Step 3: Derivation of Rates in Theorem \ref{cor:fOLS-GD}}}
\begin{proof}
We first derive the high probability bound, fixing $\gamma_n = \kappa c/(c+n)$ (where $\kappa\in(0,1)$ and $c>0$) in Theorem \ref{thm:online-high-prob} as follows:
\begin{align*}
\sum_{i=1}^{n} L_i^2
 =& \sum_{i=1}^{n}\left(\dfrac{\kappa c}{c+i}\right)^2 \prod_{j=i}^{n}\left(1 - 2 \mu \dfrac{\kappa c}{c+i}\left(1 - \dfrac{\kappa c}{c+i}\right)\right) \\
 \le & \sum_{i=1}^{n}\left(\dfrac{\kappa c}{c+i}\right)^2 \e\left(-2\mu c \kappa(1-\kappa) \sum_{j=i}^{n}\dfrac{1}{c+i}\right) \\
 \le & \frac{\kappa^2 c^2} {(n+c)^{-2\mu c \kappa(1-\kappa)}}\sum_{i=1}^{\infty} (i+c)^{-2(1-\mu c\kappa(1-\kappa))}.
\end{align*}
We now find three regimes for the rate of convergence, based on the choice of $c$ (We have used comparisons with integrals to bound the summations):\\
  \begin{inparaenum}[\bfseries(i)]
\item $\sum_{i=1}^{n} L_i^2 = O\left( (n+c)^{-2\mu c \kappa(1-\kappa)}\right)$ when $\mu c \kappa(1-\kappa)\in(0,1/2)$, \\
\item $\sum_{i=1}^{n} L_i^2 = O\left( (n+c)^{-1}\ln( n+c)\right)$ when $\mu c \kappa(1-\kappa)=1/2$, and\\
\item $\sum_{i=1}^{n} L_i^2 \le \frac{\kappa^2 c^2}{1-2(1-\mu c \kappa(1-\kappa))} (n+c)^{-1}$ when $\mu c \kappa(1-\kappa)\in(1/2,1)$.\\
\end{inparaenum}
Thus, the optimal rate for the high probability bound from Theorem \ref{thm:online-high-prob} with $( \mu c \kappa(1-\kappa) > 1/2)$ is 
\begin{align*}
P(   \l z_n \r - \E \l z_n \r \ge \epsilon ) \le \e\left(- \dfrac{\epsilon^2 (n+c)}{2 K_{\mu, c}}\right),
\end{align*}
where $K_{\mu,c}:=\kappa^2 c^2/\left(1-2(1-\mu c \kappa(1-\kappa))\right)$. 

Under the same choice of stepsize, we now bound the different error terms in Theorem \ref{thm:online-exp}.
The initial error (first term in Theorem \ref{thm:online-exp}) is bounded by $\l z_0\r n^{-\mu c \kappa}$.
The sampling error (second term in Theorem \ref{thm:online-exp}) is bounded as follows:
  \begin{align}\label{eq:high-prob-rates}
\left(\sum_{k=1}^{n} h(k)\gamma_{k}^2\e (-2\mu(\Gamma_n - \Gamma_{k}))\right)^{1/2}
\le \left(c^2 n^{-2\mu c \kappa}\sum_{k=1}^{n} h(k)(k+c)^{-2(1-\mu c \kappa)}\right)^{\frac{1}{2}} \le \sqrt{\dfrac{h(n)}{n+c}}. 
\end{align}
For bounding the drift error (third and fourth terms in Theorem \ref{thm:online-exp}), we require the following lemma:
\begin{lemma} 
Under (A1)-(A3), we have for any $\delta>0$,
$  \l \hat\theta_n - \theta^* \r \le \sqrt{\beta_n /(n\mu)}$ with probability $1-\delta$, where $\beta_n = \max\left(128 d \log n \log{n^2 \delta^{-1}}, \left(2\log{n^2 \delta^{-1}}\right)^2 \right)$.
\end{lemma}
\begin{proof}
 Follows from Theorem $5$ of \cite{dani2008stochastic} and (A3).
%  the change of norm inequality: $\lambda_{\min}(A_n)\l v \r^2 \le \vml v \vmr^2$.
\end{proof}
Using the above lemma, we bound the drift error as follows:
\begin{align*}
 &\left(\sum_{k=n_0 +1}^{n}\e\left(-2\mu(\Gamma_n - \Gamma_k)\right) \dfrac{1}{\mu^2 (k-1)^{2}}  \right)^{\frac{1}{2}}
	 + \sum_{k=n_0 +1}^{n}\e\left(-\mu(\Gamma_n - \Gamma_k)\right) \l \hat\theta_k -\theta^* \r \dfrac{1}{\mu (k-1)}\\
	 &\le \left(2 (k+c)^{-2\mu c \kappa}\sum_{k=2}^{n}(n+c)^{2 \mu c \kappa-2} \mu^{-2} \right)^{\frac{1}{2}} + 2 (n+c)^{-\mu c \kappa}\sum_{k=1}^{n} (k+c)^{\mu c \kappa - \frac{3}{2}} \mu^{-3/2}\sqrt{\beta_{k+c}}\\
	 &\le \frac{1}{\mu\sqrt{n+c}}\left(\sqrt{2} + \sqrt{\frac{\beta_{n+c}}{\mu}}\right),
\end{align*}
Thus, we have the following rate for the bound in expectation: 
\begin{align}\label{eq:expectation-error-rates}
\E \l z_n \r \le \left( \dfrac{\l z_0\r\ln(n_0)}{(n+c)^{\mu c \kappa}} + \dfrac{h(n+c)}{\sqrt{n+c}} \right).
\end{align}
Choosing $\kappa = 1/4 $ we, the claim follows from \eqref{eq:high-prob-rates} and \eqref{eq:expectation-error-rates}.
\end{proof}

\bibliography{sgd}
\bibliographystyle{plainnat}
\end{document}